\documentclass[11pt]{article}

\usepackage{acl}
\usepackage{amsmath}
\usepackage{caption}
\usepackage{subcaption}
\usepackage{graphicx}
\usepackage{amsthm}
\usepackage{booktabs}
\usepackage[ruled,vlined]{algorithm2e}
\usepackage{float} 
 \usepackage{microtype}
 \newtheorem{theorem}{Theorem}[section]      % Theorems numbered by section
\newtheorem{definition}[theorem]{Definition}

\newtheorem{remark}[theorem]{Remark}

\usepackage[english]{babel} % English as the main language
\title{Integrating Unstructured Text into Causal Inference: Empirical Evidence from Real Data}

\author{
Boning Zhou$^{1}$\thanks{\textbf{Author details.}
Boning Zhou (Amazon; work completed while at Amazon), \texttt{boningzhou0731@gmail.com}.
Ziyu Wang (Amazon), \texttt{ziywan@amazon.com}.
Han Hong (Stanford University; Amazon Scholar), \texttt{hanhon@amazon.com}.
Haoqi Hu (Amazon), \texttt{hthaoqi@amazon.com}.}
\quad
Ziyu Wang$^{1}$ \quad
Han Hong$^{1,2}$ \quad
Haoqi Hu$^{1}$ \\
\\
$^{1}$Amazon \qquad
$^{2}$Stanford University
}

\begin{document}

\maketitle
\begin{abstract}
Causal inference, a critical tool for informing business decisions, traditionally relies heavily on structured data. However, in many real-world scenarios, such data can be incomplete or unavailable. This paper presents a framework that leverages transformer-based language models to perform causal inference using unstructured text. We demonstrate the effectiveness of our framework by comparing causal estimates derived from unstructured text against those obtained from structured data across population, group, and individual levels. Our findings show consistent results between the two approaches, validating the potential of unstructured text in causal inference tasks. Our approach extends the applicability of causal inference methods to scenarios where only textual data is available, enabling data-driven business decision-making when structured tabular data is scarce.
\end{abstract}

\section{Introduction}

In industry, companies are interested in measuring the effectiveness of the launched campaign in improving revenue or how specific occurrence may harm profitability. Consider an e-commerce enterprise that launches a program to prevent counterfeits. After the launch, the company wants to evaluate the average reduction in counterfeits for products enrolled in the program, the variation in counterfeit reduction across different product groups (such as electronics, furniture, and jewelry), and the program's impact on individual enrolled products. These measurements and analysis could help assess the effectiveness of the program and identify potential areas for improvements.  

The previous mentioned problem can be regarded as event-based causal analysis which is initially formulated by \citet{neyman1923applications} and \citet{rubin1974estimating}. In this scenario, we consider a group of individuals denoted by $\{(Y_i, T_i, Z_i)\}_{i \in N}$ represents the observable data consisting copy of the random vector $(Y,T, Z)$. Each $Z_i$ represents the information characterizing individual $i$. Each individual $i$ may receive the treatment (i.e., enroll in the campaign), indicated by $T_i = 1$, or remain untreated (i.e., not enroll in the program), indicated by $T_i = 0$. The variable $Y_i$ denotes the observed outcome for individual $i$. By the nature of causal inference, for any individual $i$, we cannot observe the outcome both with and without treatment.
We define $g_0(Z_i) \equiv E[Y \mid Z = Z_i, T = 0]$ as the potential(expected) outcome for individual $i$ if untreated, and $g_1(Z_i) \equiv E[Y \mid Z = Z_i, T = 1]$ as the potential(expected) outcome if treated. Using the notation $(Z, Y, T)$, we define several treatment effect estimands commonly used in causal inference:

\begin{itemize}
    \item \textbf{Average Treatment Effect (ATE):}
    \[
    \theta_{ATE} \equiv E[g_1(Z) - g_0(Z)]
    \]
    
    \item \textbf{Average Treatment Effect on the Treated (ATET):}
    \[
    \theta_{ATET} \equiv E[g_1(Z) - g_0(Z) \mid T = 1]
    \]
    
    \item \textbf{Group Average Treatment Effect (GATE)}, for a given group $G$:
    \[
    \theta_{G} \equiv E[g_1(Z) - g_0(Z) \mid Z \in G]
    \]
    
    \item \textbf{Conditional Average Treatment Effect (CATE)} for a given individual $Z_i$:
    \[
    \theta(Z_i) \equiv g_1(Z_i) - g_0(Z_i)
    \]
\end{itemize}
To estimate these estimands, a key step is to train two models, $\tilde{g}_0(\cdot)$ and $\tilde{g}_1(\cdot)$, to predict the potential outcomes without and with treatment, respectively, given any covariate vector $Z_i \in Z$. This approach, known as the \textbf{T-learner} (Two-learner) method, was introduced by \citet{kunzel2019metalearners}. In industry applications, tabular data comprising categorical and numerical features are typically used to describe each individual $i$. When $Z_i$ is primarily structured in this form, models such as Random Forests, Boosted Trees, Neural Networks, and Penalized Regression are commonly employed to estimate $\tilde{g}_0(\cdot)$ and $\tilde{g}_1(\cdot)$; see \citet{athey2019mlguide} and \citet{nie2021quasi} for recent overviews.

However, in many real-world scenarios, tabular data are unavailable. For instance, in the counterfeit prevention campaign example, $Z_i$ may include features such as monthly sales, customer review counts, return rates, and counterfeit counts—capturing a product’s popularity and susceptibility to counterfeiting. Yet for newly listed products, such historical data are often unavailable, making causal estimation based on tabular inputs infeasible and requiring a delay until sufficient data are collected for analysis.
On the other hand, newly listed products usually have detailed text descriptions available from the first day, which are easy to access. As a result, if models $\tilde{g}_0(\cdot)$ and $\tilde{g}_1(\cdot)$ can take unstructured text as input, causal analysis can be performed promptly, enabling quicker and more informed decisions about continuing the campaign. In this paper, we leverage a pre-trained decoder-based large language model as the prediction backbone to train $\tilde{g}_0(\cdot)$ and $\tilde{g}_1(\cdot)$, which take unstructured text as input. 

Recent literature has explored the use of transformer-based models for causal inference (see \citet{kunzel2019metalearners}, \citet{maiya2021causalnlp}, \citet{zhang2022exploring}). These studies primarily rely on fully or semi-synthetic datasets to investigate how to design transformer-based models that minimize causal inference error. Consequently, they require knowledge of the ground truth causal effect, which entails making artificial assumptions about the functional relationship between potential outcomes and covariates. Due to the limitations of synthetic data, the input $Z_i$ in these works typically consists of structured features, such as numeric or categorical variables. Even when text features are included, they usually follow a regular format or exhibit a simplified numerical relationship with the generated potential outcomes.

In contrast, our paper aims to offer an empirical perspective on a fundamental question: can language models leverage unstructured text data to perform "meaningful" causal inference using real-world data? A key challenge in this exploration is that, with real-world data, the ground truth of causal effects is inherently unobservable, making standard evaluation metrics unavailable.
To address this, we take an alternative approach. For each individual $i$, we prepare both an unstructured textual description $Z_i^{text}$ and a structured representation $Z_i^{num}$ composed of numeric and categorical features. We perform causal inference using both types of input and compare the results obtained from unstructured text to those derived from structured data. If the causal estimates from unstructured text align closely with those from structured inputs—which are considered best practice in industry—we argue that meaningful causal inference can indeed be achieved using unstructured data.

The remainder of the paper is organized as follows: In Section~\ref{sec: causal estimation}, we introduce a framework for estimating treatment effects that accommodates both structured and unstructured data. Section~\ref{sec: experiment} describes our dataset and experimental design. In Section~\ref{sec: result}, we present our results, and in the final section, Section~\ref{sec: conclu}, we summarize our findings and discuss potential directions for future work.

\section{Causal Estimation}
\label{sec: causal estimation}
\subsection{Model}
We study causal inference in scenario when treatment effects are fully heterogeneous and the treatment variable is binary (\citet{chernozhukov2018double}, \citet{semenova2021debiased}). $\{(Y_i, T_i, Z_i)\}_{i \in N}$ represents the observable data consisting of i.i.d copy of the random vector $(Y,T, Z)$ having probability law $P$. 
\begin{align}
    Y &= g(T,Z) + U,\; E_{P}[U|Z,T] = 0 \label{eq: irm 1}\\
    T &= \mu(Z) + V,\; E_{P}[V|Z] = 0 \label{eq: irm 2}
\end{align}
We refer to $(Y, T, Z)$ satisfying equations \ref{eq: irm 1} and \ref{eq: irm 2} as the \textbf{Interactive Regression Model} (IRM). For notational convenience, we define $g_0(Z) \equiv g(0, Z)$ and $g_1(Z) \equiv g(1, Z)$. Additionally, standard assumption \textbf{unconfoundedness} (\citet{rosenbaum1983central}) is made:
\[
(Y_1 \equiv g_1(Z) + U, Y_0 \equiv g_0(Z) + U) \perp T|Z
\]
\subsection{Estimation}
Given trained prediction models $\tilde{g}_0(\cdot)$, $\tilde{g}_1(\cdot)$ and $\tilde{\mu}(\cdot)$ for unknown $g_0(\cdot)$, $g_1(\cdot)$ and $\mu(\cdot)$. We construct the corresponding estimators for $\theta_{ATE}$, $\theta_{ATET}$ and $\theta_{G}$ following the work by \citet{chernozhukov2018double}:  
\begin{align*} 
&\tilde{\theta}_{ATE} = E_N[\tilde{g}_1(Z_i) - \tilde{g}_0(Z_i) + (Y_i - \tilde{g}_{T_i}(X))\tilde{H}_i ]\\
&\tilde{\theta}_{ATET} = E_N[\frac{D_i}{E_N[T_m]}(\tilde{g}_1(Z_i) - \tilde{g}_0(Z_i)) + \\
&\frac{\tilde{\mu}(Z_i)}{E_{N}[T_m]}(Y_i - \tilde{g}_{T_i}(X))\tilde{H}_i]\\
&\tilde{\theta}_G = E_N[\frac{\mathrm{1}(Z_i \in G)}{E_N[\mathrm{1}(Z_m \in G)]}(\tilde{g}_1(Z_i) - \tilde{g}_0(Z_i) +\\
&(Y_i - \tilde{g}_{T_i}(X))\tilde{H}_i)]
\end{align*}
where
\[
\tilde{H}_i \equiv \frac{T_i}{\tilde{\mu}(Z_i)} - \frac{1 - T_i}{1 - \tilde{\mu}(Z_i)}
\]
$E_N[\cdot]$ denotes the sample average. In particular, $E_N[T_m] \in (0,1)$ represents the proportion of treated individuals in the observed sample, and $E_N[\mathrm{1}(Z_m \in G)] \in (0,1)$ denotes the proportion of individuals in group $G$ within the observed sample. 

For the estimator of $\theta(Z_i)$, we adopt agnostic approach as \citet{chernozhukov2018generic} and estimate the following estimands for each $Z_i$: 
\[
E[\theta(Z_i)|\tilde{\theta}(Z_i)]
\]
Where $\tilde{\theta}(Z_i) \equiv \tilde{g}_1(Z_i) - \tilde{g}_0(Z_i)$. These estimands represent the best linear projection (\textbf{BLP}) of the true \textbf{CATE} $\theta(Z_i)$ onto the model’s naive prediction $\tilde{\theta}(Z_i)$. \citet{chernozhukov2018generic} propose an estimator for the estimands under a randomized controlled trial (RCT) setting, where $\mu(\cdot)$ is assumed to be known. However, in industry applications, RCTs are often difficult to implement, and we typically rely on observational data, where $\mu(\cdot)$ is unknown. To address this challenge, we propose a modified estimator: For the observational data ${(Y_i, T_i, Z_i)}_{i \in N}$, given $\tilde{\eta} \equiv (\tilde{g}_0(\cdot), \tilde{g}_1(\cdot), \tilde{\mu}(\cdot))$, we construct the doubly robust label for each $i \in N$ as
\begin{equation}
c_{\tilde{\eta}}(Z_i) \equiv \tilde{g}_1(Z_i) - \tilde{g}_0(Z_i) + (Y_i - \tilde{g}_{T_i}(X))\tilde{H}_i  
\label{eq: double robust label}
\end{equation}

We then run an OLS regression of $c_{\tilde{\eta}}(Z_i)$ on $(1, \tilde{\theta}(Z_i))$ to obtain coefficients $\tilde{a}_1$ and $\tilde{b}_1$. Finally, we use
\begin{equation}
\label{eq: cate generate}
\tilde{a}_1 + \tilde{b}_1\tilde{\theta}(Z_i)      
\end{equation}

as an estimate of $E[\theta(Z_i) \mid \tilde{\theta}(Z_i)]$. Please see Appendix \ref{sec:appendix} for details of this CATE estimate. 
\subsection{Algorithm}
In real-world scenarios, our initial data input consists only of ${(Y_i, D_i, Z_i)}_{i \in N}$, without access to $(\tilde{g}_0(\cdot), \tilde{g}_1(\cdot), \tilde{\mu}(\cdot))$. While we can train models on one subset of the data and use them to make predictions on another, this approach may lead to limited coverage. To address this issue, we employ cross-fitting, which allows for more reliable estimates by rotating the roles of training and estimation across data folds.
\begin{algorithm}[ht]
\caption{Estimate ATE, ATET, GATE and CATE via Cross-Fitting}
\label{alg:cate_estimation}

\KwIn{Observational data $\{(Y_i, D_i, Z_i)\}_{i \in N}$, number of folds $K$}
\KwOut{ATE, ATET, GATE on the total population and CATE estimates for each $i \in N$.}

Partition the data into $K$ equal-sized folds: $\bigcup_{k=1}^K I_k$\;
ScoreList = [\quad]\;
\For{$k = 1$ \KwTo $K$}{
    Let $I_{-k} := \bigcup_{j \notin I_k} I_j$\;

    Train ML models on $I_{-k}$ to obtain $\tilde{\eta}_k := (\tilde{g}_1^k (\cdot), \tilde{g}_0^k(\cdot), \tilde{\mu}^k(\cdot))$\;

    \ForEach{$i \in I_k$}{
        Compute $c_{\tilde{\eta}_k}(Z_i)$ as defined in Equation~\eqref{eq: double robust label}\;
        
        Compute $\tilde{\theta}_k(Z_i)$ \;

        ScoreList.append($(\tilde{g}_1^k (Z_i), \tilde{g}_0^k(Z_i), \tilde{\mu}^k(Z_i))$)\;
    }

    Computer $E_k[\tilde{\theta}_k(Z_i)]$ , Regress $c_{\tilde{\eta}_k}(Z_i)$ on $(1, \tilde{\theta}_k(Z_i) - E_k[\tilde{\theta}_k(Z_i)])$ to obtain $(\tilde{a}_1^k, \tilde{b}_1^k)$ Apply $(\tilde{a}_1^k, \tilde{b}_1^k)$ to estimate CATE for $i \in I_k$ by equation \eqref{eq: cate generate}\;
}
Use ScoreList to calculate the $\tilde{\theta}_{ATE}$, $\tilde{\theta}_{ATET}$ and $\tilde{\theta}_{G}$ based on their estimator. 
\end{algorithm}

\section{Experiment Design}
\label{sec: experiment}
In our experiment, we apply the proposed framework from Section~\ref{sec: causal estimation} to estimate the causal impact of an anti-counterfeit campaign program. For each $i \in N$, we observe two types of information that characterize the product: $Z_i^{\text{text}}$ and $Z_i^{\text{num}}$. The former refers to the unstructured textual description of the product and the latter represents structured tabular data for the same product. We train separate models using the text features and numerical features, respectively, and obtain causal effect estimates from each. We then compare the results across the two modalities to assess their consistency at the population, product group, and individual levels.

\subsection{Data}

Within the overall product population, approximately 280,000 products are enrolled in the program (treatment group), while 150,000 are not (control group). The control group is constructed using a matching method following \citet{lanners2023variable}\footnote{The matching process ensures that the positivity assumption, $0<P(T = 1 | Z) < 1$, holds\cite{rosenbaum1983central}}, identifying products from the broader product space that are similar to those in the treatment group.

Adopting the notation from the previous section, the total population is represented by ${(Y_i, T_i, Z_i)}_{i \in N}$. Here, $Y_i \in (0,1)$ denotes the complaint rate for product $i$ within a specified time window following the initiation of the program. The treatment indicator $T_i = 1$ indicates that the product is enrolled in the program, while $T_i = 0$ indicates it is not. Since the program is designed to reduce counterfeit incidents, we expect it to lower the complaint rate for products enrolled in the program. For each $i \in N$, we observe two types of information $Z_i$ that characterize the product. In particular, $Z_i^{\text{text}}$ represents the unstructured textual description of the product, typically composed of sentences or short phrases without a standardized format. See Figure~\ref{fig:synthetic_example} for examples of two synthetic products.\footnote{The example shown above is synthetically generated for illustrative purposes. It follows the same pattern as the real training data (e.g., short description, bullet points), the content itself is entirely fictional and does not correspond to any actual product.} On the other hand, $Z_i^{\text{num}}$ represents tabular data comprising a product’s historical records, including sales, reviews, popularity, pricing, complaints, and other attributes—totaling over 150 features. These numerical and categorical features are selected based on expert insights and engineered by data professionals, and can be regarded as informative variables.

\begin{figure}[H]  % or [htbp] if you want LaTeX to decide placement
    \centering
    \includegraphics[width=5cm, height = 5cm]{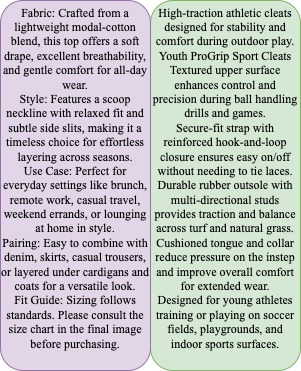}
    \caption{Illustrative example of unstructured texts}
    \label{fig:synthetic_example}
\end{figure}

\subsection{Model Training}
For each $k \in \{1,\dots, K\}$, we train $\tilde{\eta}_k^{text} := (\tilde{g}_1^{k, text} (\cdot), \tilde{g}_0^{k, text}(\cdot), \tilde{\mu}^{k,text}(\cdot))$ and $\tilde{\eta}_k^{num} := (\tilde{g}_1^{k, num} (\cdot), \tilde{g}_0^{k, num}(\cdot), \tilde{\mu}^{k,num}(\cdot))$ on $I_{-k}$ based on $Z_k^{text}$ and $Z_k^{num}$ separately.  

Based on the dataset ${(Y_i, T_i, Z_i^{\text{text}})}_{i \in I_{-k}}$, we employ a single transformer-based model to simultaneously generate
\[
(\tilde{g}_1^{k,\text{text}}(Z_i^{text}), \tilde{g}_0^{k,\text{text}}(Z_i^{text}), \tilde{\mu}^{k,\text{text}}(Z_i^{text})) 
\]
for each $i$. To address this multi-output regression task, we use a pre-trained decoder-only language model and apply supervised fine-tuning. We opt for a decoder-only architecture, rather than an encoder or encoder-decoder model, due to its higher data efficiency and superior scalability in training, as demonstrated in \citet{yang2024harnessing}, \citet{touvron2023llama}, and \citet{jiang2023mistral}. Additionally, we adopt the Mistral-3B (Mini Mistral) model\footnote{The Mistral 3B model used in this research is licensed under the Apache License 2.0. The full license text can be found at https://choosealicense.com/licenses/apache-2.0/}—a distilled version of the latest Mistral-7B(\citet{jiang2023mistral}) \footnote{We initially aimed to use a 7B model for regression, but training it on 40GB A100s—even with LoRA—requires FSDP, adding complexity and overhead. Instead, we used Mistral’s 3B distilled model, which trains efficiently with DDP and offers performance broadly comparable to the 7B version, enabling faster experimentation while maintaining accuracy.}—as our predictive backbone. We extract the last hidden state of the last effective token from the input sequence in $Z_i^{\text{text}}$ and pass it through an autoencoder to perform the multi-target regression task. We backpropagate through both the autoencoder and the Mistral model, applying LoRA to fine-tune 1\% of the Mistral model’s parameters. An additional reason we use a mainstream pre-trained language model instead of building a transformer-based model from scratch is our belief that the model’s embedded world knowledge can facilitate transfer learning for our causal task. Specifically, product descriptions often convey information about a product’s popularity and its vulnerability to counterfeiting. We aim to leverage this implicit knowledge through transfer learning and fine-tune the model to perform causal inference in our setting. The loss function for the multi-label regression is:

{\small
\begin{align*}
    &\text{Loss}(Z_i^{text}) = \\
    &\quad\frac{\sqrt{T_i\left(\tilde{g}_1^{k,\text{text}}(Z_i^{text}) - Y_i\right)^2}}{E_k[Y_m]} + \\
    &\quad\frac{\sqrt{(1-T_i)\left(\tilde{g}_0^{k,\text{text}}(Z_i^{text}) - Y_i\right)^2}}{E_k[Y_m]} +\\
    &\quad \lambda \cdot BCE(\tilde{\mu}^{k,\text{text}}(Z_i^{text}), T_i)
\end{align*}
}

 Given an individual product $i \in I_{-k}$, if it is treated, we compute the loss only between the predicted outcome under treatment and the observed outcome. Conversely, if it is not treated, we compute the loss only between the predicted outcome without treatment and the observed outcome. The final term is the standard binary cross-entropy loss used to train the classifier for predicting the probability of treatment. The parameter $\lambda$ serves as a weighting factor to balance the loss between the regression and classification tasks.\footnote{The reason we use a single model to predict both potential outcomes and the propensity score is to leverage the LLM’s multitask capability and to avoid the computational cost of training separate models for each prediction.}

For tabular data, the training logic is straightforward. We use the subset $\{(Z_i^{\text{num}}, Y_i)\}_{i \in I{-k}, T_i = 1}$ to train the treated outcome model $\tilde{g}_1^{k, \text{num}}(\cdot)$, and $\{(Z_i^{\text{num}}, Y_i)\}_{i \in I_{-k}, T_i = 0}$ to train the control outcome model $\tilde{g}_0^{k, \text{num}}(\cdot)$. The propensity score model $\tilde{\mu}^{k, \text{num}}(\cdot)$ is trained on $\{(Z_i^{\text{num}}, T_i)\}_{i \in I_{-k}}$. We consider various candidate models\footnote{We try various models, including Lasso, Elastic Net, Gradient Boost, Random Forest, and Deep Neural Nets, and display the best four in this paper.} for tabular data and the final model is selected based on performance of $\tilde{g}_1^{k, \text{num}}$ and $\tilde{g}_0^{k, \text{num}}$.

\section{Results}
\label{sec: result}
For the set of numerical models $\{(\tilde{g}_1^{k,\text{num}}, \tilde{g}_0^{k,\text{num}})\}_{k \in K}$ trained on different learners, gradient boosting demonstrates the best overall performance. Table~\ref{table: numerical comparison} summarizes the comparative results, with the definitions of the reported metrics provided in Appendix~\ref{app: metric}. As shown in the table, gradient boosting achieves the highest out-of-sample correlation in both the treatment and control groups, along with the lowest MAPE. Therefore, we adopt gradient boosting as the predictive backbone for conducting causal inference using numerical data.
\begin{table}[H]
  \centering
  \small
  \begin{tabular}{lccc}
    \hline
    \textbf{Method} & \textbf{Corr T} & \textbf{Corr C} & \textbf{MAPE} \\
    \hline
    Gradient Boost & 0.75  & 0.70  & 0.47 \\
    Random Forest     & 0.73  & 0.68  & 0.50 \\
    Lasso             & 0.58  & 0.45  & 0.67 \\
    Elastic Net       & 0.58  & 0.46 & 0.66 \\
    \hline
  \end{tabular}
  \caption{Comparison: Model Performance}
\label{table: numerical comparison}
\end{table}
For the text data, LLM achieve 0.64 out-of-sample correlation in treatment group and 0.58 in control group and overall MAPE is 0.55. Figure~\ref{fig:dml_llm_2x2} in Appendix~\ref{app: control treat graph} presents the distributions of predicted and actual outcomes for both the treatment and control groups, using Gradient Boosting and the LLM-based models, respectively.

As shown in Table \ref{table:ATE} and Figure \ref{fig:boxplot_ate_att} (all numbers in Table \ref{table:ATE} and Figure \ref{fig:boxplot_ate_att} are \%), ATE and ATT achieved using text data and numeric data are highly similar with confidence intervals largely overlapped.

\begin{table}[h]
  \centering
  \small
  \begin{tabular}{lcccc}
    \hline
      & \textbf{ATE} & \textbf{ATE CI} & \textbf{ATT} & \textbf{ATT CI} \\
    \hline
     Num     & -0.94  & [-1, -0.9]  & -1.1 & [-1.2, -0.9]  \\
     Text     & -1.1   &  [-1.2, -0.9] & -1.2 & [-1.4, -1]  \\\hline
  \end{tabular}
  \caption{Comparison: ATE \& ATT}
  \label{table:ATE}
\end{table}

\begin{figure}[h]
  \centering
  \includegraphics[width=0.48\textwidth]{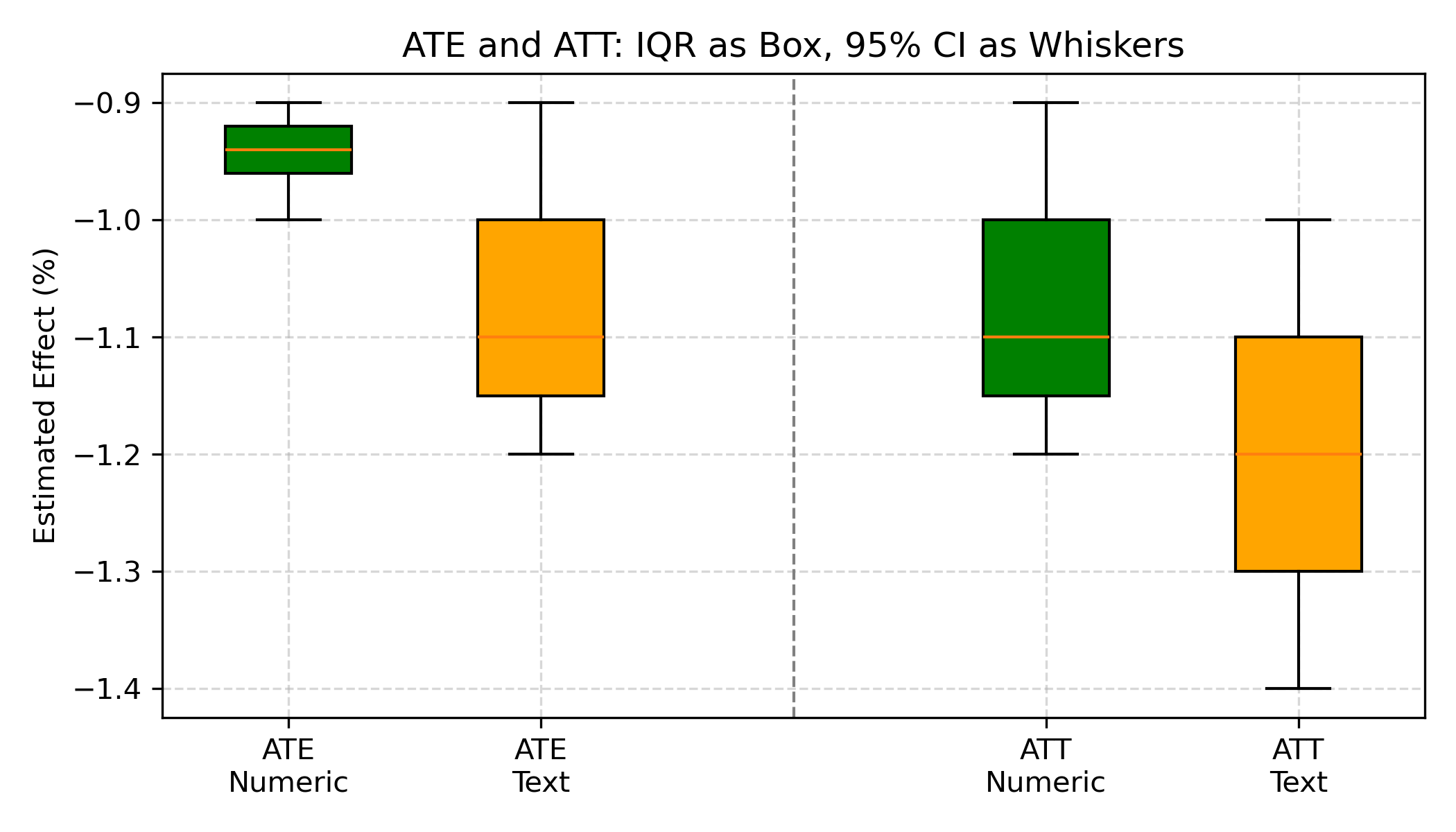}
  \caption{ATE \& ATT Comparison: Box plots}
  \label{fig:boxplot_ate_att}
\end{figure}

We take a further step to compare the Group Average Treatment Effects (GATEs) across different product groups. We find that estimates derived from text data are highly consistent with those from numeric data, with Pearson correlation of 0.92 and Spearman correlation of 0.78. As shown in Table~\ref{table:complaint_reduction_sorted}, rankings of product groups by average complaint rate reduction from numeric/text data are consistent. Figure~\ref{fig-group} in Appendix \ref{app: GATE plots} visualizes the GATEs by product group, illustrating a strong alignment between the two sets of estimates.

\begin{table}[h]
  \centering

  \small
  \begin{tabular}{@{}lcc@{}}
    \toprule
    \textbf{Group} & \textbf{Numeric Rank} & \textbf{Text Rank} \\
    \midrule
    Group\_A           & 1  & 1  \\
    Group\_B            & 2  & 2  \\
    Group\_C           & 3  & 4  \\
    Group\_D         & 4  & 3  \\
    Group\_E               & 5  & 8  \\
    Group\_F           & 6  & 5  \\
    Group\_G          & 7  & 12 \\
    Group\_H             & 8  & 15 \\
    Group\_I              & 9  & 9  \\
    Group\_J              & 10 & 7  \\
    Group\_K               & 11 & 11 \\
    Group\_L            & 12 & 14 \\
    Group\_M            & 13 & 10 \\
    Group\_N        & 14 & 19 \\
    Group\_O      & 15 & 13 \\
    Group\_P      & 16 & 20 \\
    Group\_Q          & 17 & 6  \\
    Group\_R              & 18 & 16 \\
    Group\_S       & 19 & 17 \\
    Group\_T    & 20 & 18 \\
    \bottomrule
  \end{tabular}
    \caption{Product Group Rankings}
  \label{table:complaint_reduction_sorted}
\end{table}

Lastly, we compare individual-level CATEs estimated from textual and numerical data and find that the two sets of estimates are also highly correlated, with Pearson correlation of 0.60 and Spearman correlation of 0.68. In Figure~\ref{fig-HTE_Curve}, we rank all 430,000 products from lowest to highest based on their estimated CATEs from text and numeric data ($y$-axis represents complaint rate reduction). This figure visualizes the quantile distribution of the two sets of CATEs and the two quantile curves are well aligned, indicating strong consistency between the text-based and numeric-based estimates.

\begin{figure}[h]
  \centering
  \includegraphics[width=\columnwidth]{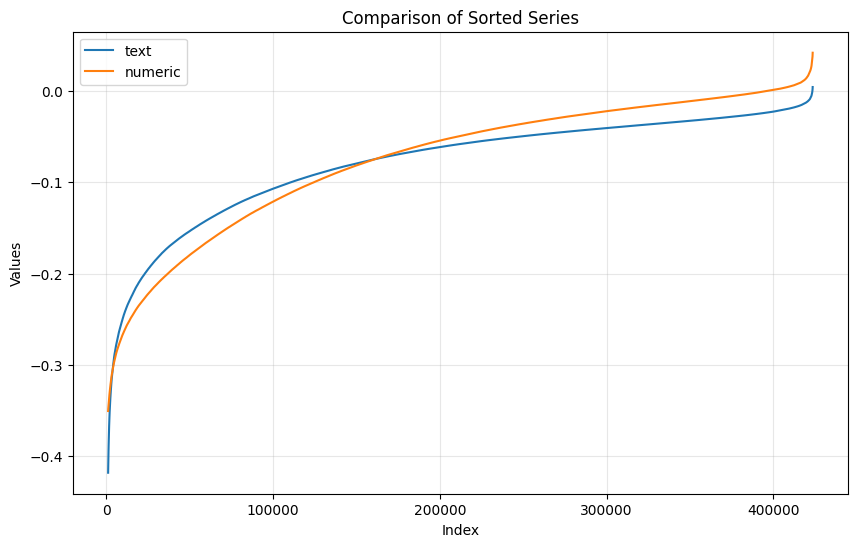}
  \caption{CATE Comparison: CATE Curves}
  \label{fig-HTE_Curve}
\end{figure}

Additionally, we evaluate the CATE estimated from text data as a recommendation score. Using the CATE derived from numerical data as the ground truth, we construct a lift curve based on both estimates. In Figure \ref{fig-Lift_Curve}, the red curve ranks all products from high to low using the text-based CATE, while the $y$-axis represents the cumulative gain computed using the numeric-based CATE. The blue curve, by contrast, uses the numeric-based CATE for both sorting and cumulative gain calculation—representing the optimal lift curve achievable. We find that the area under the red curve exceeds 84\% of the area under the blue curve, indicating that the text-derived CATE is highly effective as a recommendation score.

\begin{figure}[h]
  \centering
  \includegraphics[width=0.8\columnwidth]{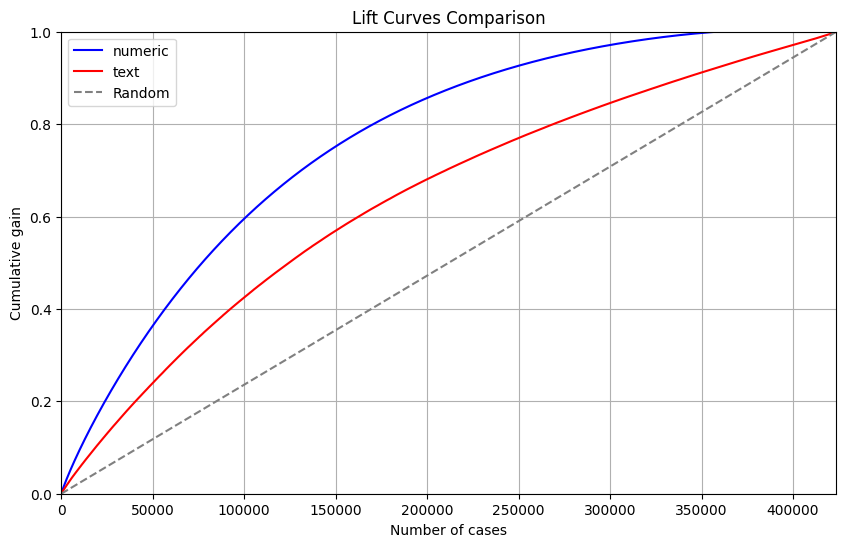}
  \caption{CATE Comparison: Lift Curves}
  \label{fig-Lift_Curve}
\end{figure}

From the above, we conclude that the LLM is able to estimate causal impacts from unstructured text data that are comparable to those derived from structured numeric data.

\section{Conclusion \& Extension}
\label{sec: conclu}

In this work, we present a framework to leverage transformer-based language models for causal inference using unstructured text. Our approach extends the applicability of causal inference methods to scenarios where only textual data is available and we demonstrate the effectiveness of our framework by comparing causal estimates derived from unstructured text against those obtained from tabular data. Our findings show consistent results between the two approaches, validating the potential of unstructured text in causal inference tasks.

Our proposed method currently applies to the single binary treatment setting. In real-world practice, however, treatment assignment can be continuous (e.g., treatment intensity) rather than binary, and in some cases individuals may receive multiple treatments at the same or staggered times. Extending our framework to continuous or multi-treatment scenarios is a meaningful direction for future research. Another promising direction is multi-modal training. In this paper, we show that text-based inputs alone can achieve prediction and estimation performance comparable to that obtained using structured numerical data. However, it remains an open question whether—and how—jointly training on both textual and numerical data can further improve predictive accuracy and yield estimates with higher confidence and greater statistical power. 

\newpage 
\bibliography{custom}

\appendix

\section{Derivation of CATE}
\label{sec:appendix}
In the note, the estimands of interest is \textbf{Best Linear Predictor of CATE} (BLP) from (\cite{chernozhukov2018generic}). Given ML proxy $\tilde{g}_1(\cdot), \tilde{g}_0(\cdot)$ and $\tilde{\mu}(\cdot)$ for $g_1(\cdot), g_0(\cdot)$ and $\mu(\cdot)$ and $(Y,T,Z)$. We interested in:
\begin{definition}(BLP)
    The best linear predictor of $\theta(Z) \equiv g_1(Z) - g_0(Z)$ by $\tilde{\theta}(Z):= \Tilde{g}_1(Z) - \Tilde{g}_0(Z)$ is the solution to: 
    \[
    \min_{b_1, b_2} E[\theta(Z) - b_1 - b_2\tilde{\theta}(Z)]^2
    \]
    which if exists, is defined as
    \[
    BLP[\theta(Z)|\tilde{\theta}(Z)] := \beta_1 + \beta_2(\tilde{\theta}(Z) - E[\tilde{\theta}(Z)])
    \]
    where $\beta_1 = E[\tilde{\theta}(Z)]$ and $\beta_2 = \frac{Cov(\theta(Z), \tilde{\theta}(Z))}{Var[\tilde{\theta}(Z)]}$
\end{definition}

In observation data setting, directly regress Horvitz-Thompson transformation $HY$, where $H = \frac{T - \mu(Z)}{\mu(Z)(1-\mu(Z))}$, on $\tilde{\theta}(Z)$ is not feasible since we do not know the true value $\mu(Z)$. Moreover, if we replace $\mu(\cdot)$ with $\Tilde{\mu}(\cdot)$ and regress $\Tilde{H}Y$, where $\Tilde{H} = \frac{T - \Tilde{\mu}(Z)}{\Tilde{\mu}(Z)(1-\Tilde{\mu}(Z))}$, on $\tilde{\theta}(Z)$. The bias of estimator will solely depends on $\Tilde{\mu}(Z)$. As an alternative, we propose another identification strategy follow double robust spirit and show its good property.

\textbf{Strategy: Double Robust Label:}
Given $(Y, T, Z)$ with probability law $P$ and ML proxy $\Tilde{\eta} \equiv (\Tilde{g}_1(\cdot), \Tilde{g}_0(\cdot), \Tilde{\mu}(\cdot))$, construct the following label (random variable):
\[
c(\Tilde{\eta}) = \Tilde{H} (Y - \Tilde{g}_T(Z)) + \Tilde{g}_1(Z) - \Tilde{g}_0(Z)
\]
\[
\tilde{H} = \frac{T}{\Tilde{\mu}(Z)} - \frac{1-T}{1 - \tilde{\mu}(Z)}
\]

\begin{theorem}
\label{theorem 1}
    Given $\tilde{\eta}$ with assuming $\tilde{g}(1, Z)$ and $\tilde{g}(0, Z)$ have finite second moment and do not equal to each other almost surely. Then $(a_1, b_1)$ by considering following linear projection, 
    \begin{equation}
            c(\tilde{\eta}) = a_1 + b_1(\tilde{\theta}(Z) - E[\tilde{\theta}(Z)]) + \epsilon
    \label{linear project}
    \end{equation}
    with $\; E[\epsilon \tilde{X}] = 0\;, \tilde{X} = [1, \tilde{\theta}(Z) - E[\tilde{\theta}(Z)]]'$\\
    has the following expression:
    \[
    a_1 = \beta_1 + E[bias_1] + E[bias_2]
    \]
    \[
    b_1 = \beta_2 + \frac{cov(bias_1, \tilde{\theta}(Z))}{Var(\tilde{\theta}(Z))} + \frac{cov(bias_2, \tilde{\theta}(Z))}{Var(\tilde{\theta}(Z))}
    \]
    With $bias_1$ and $bias_2$ are two random variable:
    \[
    bias_1 = (\frac{\mu_(Z)}{\tilde{\mu}(Z)} - 1)(g_1(Z) - \tilde{g}_1(Z))
    \]
    \[
    bias_2 = (1 - \frac{ 1 - \mu_(Z)}{1 - \tilde{\mu}(Z)})(g_0(Z) - \tilde{g}_0(Z))
    \]
\end{theorem}
\begin{proof}
    The Proof is straightforward once we decompose the expression of $c(\tilde{\eta})$ \footnote{In the proof, we will use $g(T, X)$ and $g_T(X)$, as well as $\tilde{g}(T, X)$ and $\tilde{g}_T(X)$, interchangeably for convenience.}.
    \begin{align*}
        c(\Tilde{\eta}) &= \Tilde{H} (Y - \Tilde{g}(T,Z)) + \Tilde{g}(1,Z) - \Tilde{g}(0,Z)\\
        &= (\frac{T}{\tilde{\mu}(Z)} - \frac{1-T}{1 - \tilde{\mu}(Z)}) (Y - \Tilde{g}(T,Z)) + \\
        &\quad \Tilde{g}(1,Z) - \Tilde{g}(0,Z)\\
        &= (\frac{T}{\tilde{\mu}(Z)} - \frac{1-T}{1 - \tilde{\mu}(Z)})(g(T,Z)\\
        &\quad -\tilde{g}(T,Z)) - (g(1,Z) - g(0,Z))\\
        &\quad + (g(1,Z) - g(0,Z)) \\
        & \quad + (\tilde{g}(1,Z) - \tilde{g}(0,Z)) + \tilde{H}U\\
        & = (g(1,Z) - g(0,Z)) - (g(1,Z)\\
        & \quad - \tilde{g}(1,Z)) + (g(0,Z) - \tilde{g}(0,Z)) \\
        & \quad + (\frac{T}{\tilde
        \mu(Z)} \quad- \frac{1-T}{1-\tilde
        \mu(Z)})(g(T,Z) \\
        &\quad - \tilde{g}(T,Z)) + \tilde{H}U\\
        &= s_0(Z) + (\frac{T}{\tilde{\mu}(Z)} - 1)(g(1,Z) - \tilde{g}(1,Z))\\
        &\quad + (1 - \frac{1-T}{1 - \tilde{\mu}(Z)})(g(0,Z)\\
        & \quad - \tilde{g}(0,Z)) + \tilde{H}U
    \end{align*}
    By letting $a = (a_1, b_1)$, the normal equation(first order condition) for regression formula \ref{linear project} is 
    \[
    E[(c(\tilde{\eta}) - a'\tilde{X})\tilde{X}] = 0
    \]
    and since $E[U|T,Z] = 0$, we can omit the $\tilde{H}U$ in the normal equation
    \begin{align*}
        &E[(s_0(Z) + (\frac{E[T|Z]}{\tilde{\mu}(Z)} - 1)(g(1,Z) - \tilde{g}(1,Z)) + \\
        &\quad (1 - \frac{1-E[T|Z]}{1 - \tilde{\mu}(Z)})(g(0,Z) - \tilde{g}(0,Z)) - a'\tilde{X})'\Tilde{X}] \\
        &\quad = 0
    \end{align*}
    since $1$ and $\tilde{\theta}(Z) - E[\tilde{\theta}(Z)]$ are orthogonal, we will get the solution in theorem. 
\end{proof}
\begin{remark}
    From theorem \ref{theorem 1}, we can see the $bias_1$ and $bias_2$ have double robust property meaning either potential outcome function or propensity score has the correct form, $a_1, b_1$ are equal to $\beta_1, \beta_2$. Moreover, even they are both mis-specified, for example, $(\tilde{g}(T,Z) - g(T,Z)) \sim O(N^{-\frac{1}{4}}) $ and $(\mu(X) - \tilde{\mu}(X)) \sim O(N^{-\frac{1}{4}})$. $bias_1$ and $bias_2$ will have higher convergence rate $O(N^{-\frac{1}{2}})$ which achieve higher efficiency. 
\end{remark}
\section{Metric definition for Table \ref{table: numerical comparison}}
\label{app: metric}
For different method, we compare the correlation between 
\[
\bigcup_{k \in K}\{ \tilde{g}_1^{k}(Z_i)|i \in I_k, T_i = 1\}
\]
and 
\[
\bigcup_{k \in K}\{ Y_i|i \in I_k, T_i = 1\}
\]
We also compare the correlation between 
\[
\bigcup_{k \in K}\{ \tilde{g}_0^{k}(Z_i)|i \in I_k, T_i = 0\}
\]
and 
\[
\bigcup_{k \in K}\{ Y_i|i \in I_k, T_i = 0\}
\]
And also the mean average percentage error between
\[
\bigcup_{k \in K}\{ \tilde{g}_{T_i}^{k}(Z_i)|i \in I_k\}
\]
and 
\[
\bigcup_{k \in K}\{ Y_i|i \in I_k\}
\]
We name the first two correlations as \textbf{Treatment Group Correlation(Corr T)} and \textbf{Control Group Correlation(Corr C)}. Since both $\tilde{g}_{1}^{k}(\cdot)$ and $\tilde{g}_{0}^{k}(\cdot)$ are trained on $I_{-k}$ and test on $I_k$, so the correlations defined above represent the out of sample performance. 

\section{Graph for Corr T and Corr C}
\label{app: control treat graph}

\begin{figure}[ht]
  \centering
  \begin{minipage}{0.49\linewidth}
    \centering
    \includegraphics[width=\linewidth]{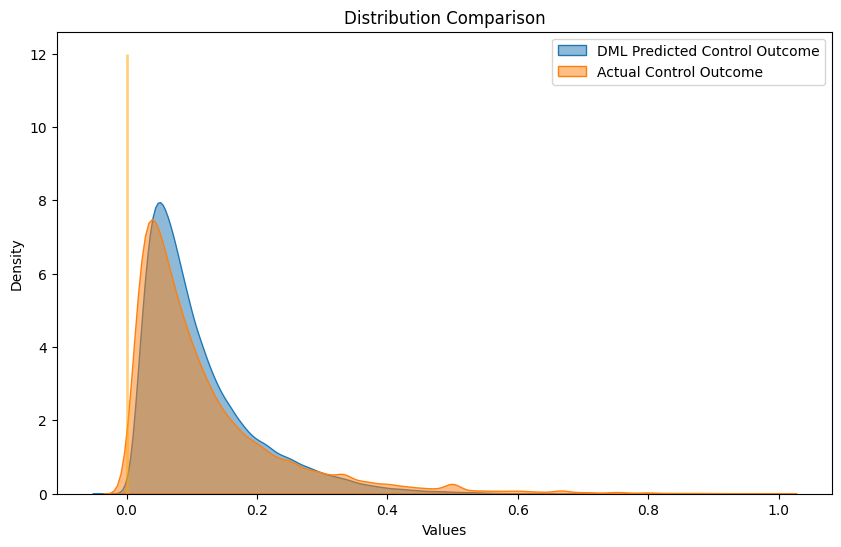}
    \subcaption{Gradient Boost Control}
    \label{subfig: a}
  \end{minipage}
  \hfill
  \begin{minipage}{0.49\linewidth}
    \centering
    \includegraphics[width=\linewidth]{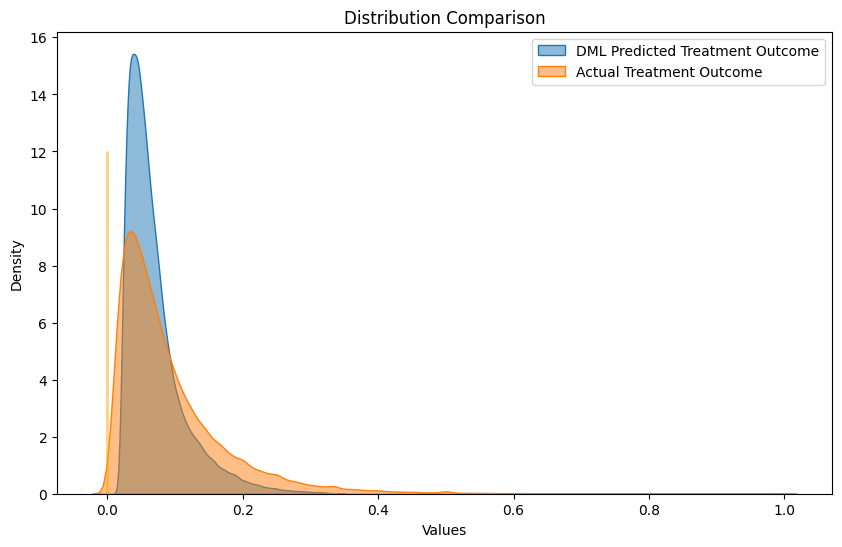}
    \subcaption{Gradient Boost Treatment}
    \label{subfig: b}
  \end{minipage}

  \vspace{0.5em}

  \begin{minipage}{0.49\linewidth}
    \centering
    \includegraphics[width=\linewidth]{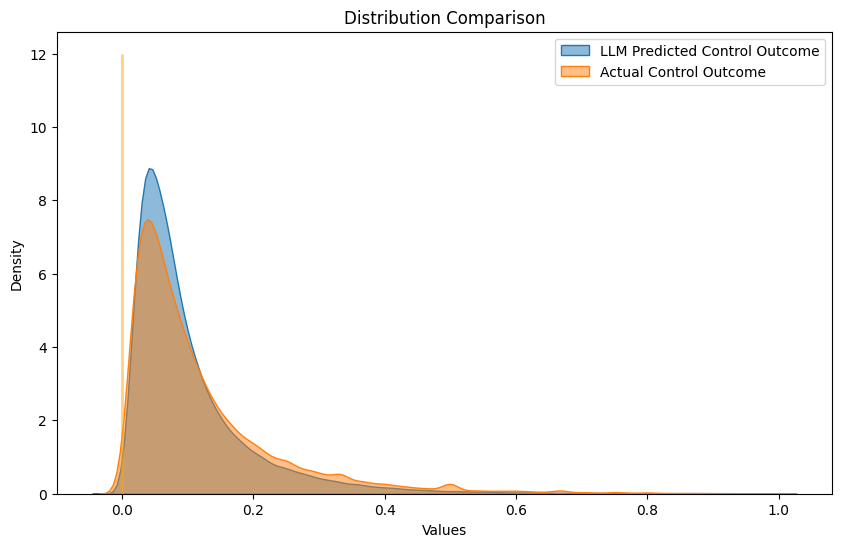}
    \subcaption{LLM Control}
    \label{subfig: c}
  \end{minipage}
  \hfill
  \begin{minipage}{0.49\linewidth}
    \centering
    \includegraphics[width=\linewidth]{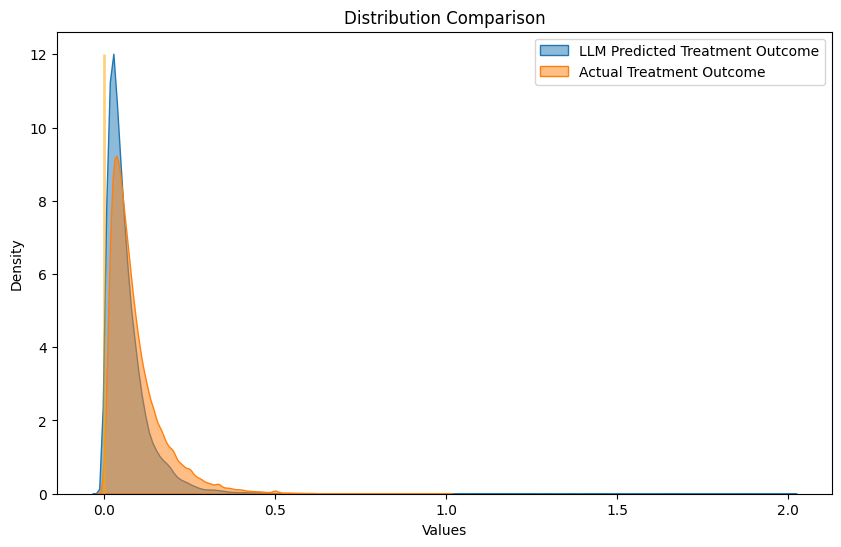}
    \subcaption{LLM Treatment}
    \label{subfig: d}
  \end{minipage}

  \caption{Gradient Boosting/LLM outcome prediction performance on control and treatment groups.}
  \label{fig: }
  \label{fig:dml_llm_2x2}
\end{figure}
Figure \ref{subfig: a} plots distribution of \[
\bigcup_{k \in K}\{ \tilde{g}_0^{k}(Z_i^{num})|i \in I_k, T_i = 0\}
\]
and 
\[
\bigcup_{k \in K}\{ Y_i|i \in I_k, T_i = 0\}
\]
Figure \ref{subfig: b} plots distribution of \[
\bigcup_{k \in K}\{ \tilde{g}_1^{k}(Z_i^{num})|i \in I_k, T_i = 1\}
\]
and 
\[
\bigcup_{k \in K}\{ Y_i|i \in I_k, T_i = 1\}
\]
Given $\tilde{g}_0^{k}(\cdot)$ and $\tilde{g}_1^{k}(\cdot)$ are trained by Gradient Boosting method. 
Figure \ref{subfig: c} plots distribution of \[
\bigcup_{k \in K}\{ \tilde{g}_0^{k}(Z_i^{text})|i \in I_k, T_i = 0\}
\]
and 
\[
\bigcup_{k \in K}\{ Y_i|i \in I_k, T_i = 0\}
\]
Figure \ref{subfig: d} plots distribution of \[
\bigcup_{k \in K}\{ \tilde{g}_1^{k}(Z_i^{text})|i \in I_k, T_i = 1\}
\]
and 
\[
\bigcup_{k \in K}\{ Y_i|i \in I_k, T_i = 1\}
\]
Given that $\tilde{g}_0^{k}(\cdot)$ and $\tilde{g}_1^{k}(\cdot)$ are trained using the LLM on unstructured text data, the graph shows that their out-of-sample performance is comparable to that of $\tilde{g}_0^{k}(\cdot)$ and $\tilde{g}_1^{k}(\cdot)$ trained with Gradient Boosting on structured numerical data.

\section{GATEs for Different Groups}
\label{app: GATE plots}

\begin{figure}[h]
  \centering
  \includegraphics[width=5.5cm, height =4cm]{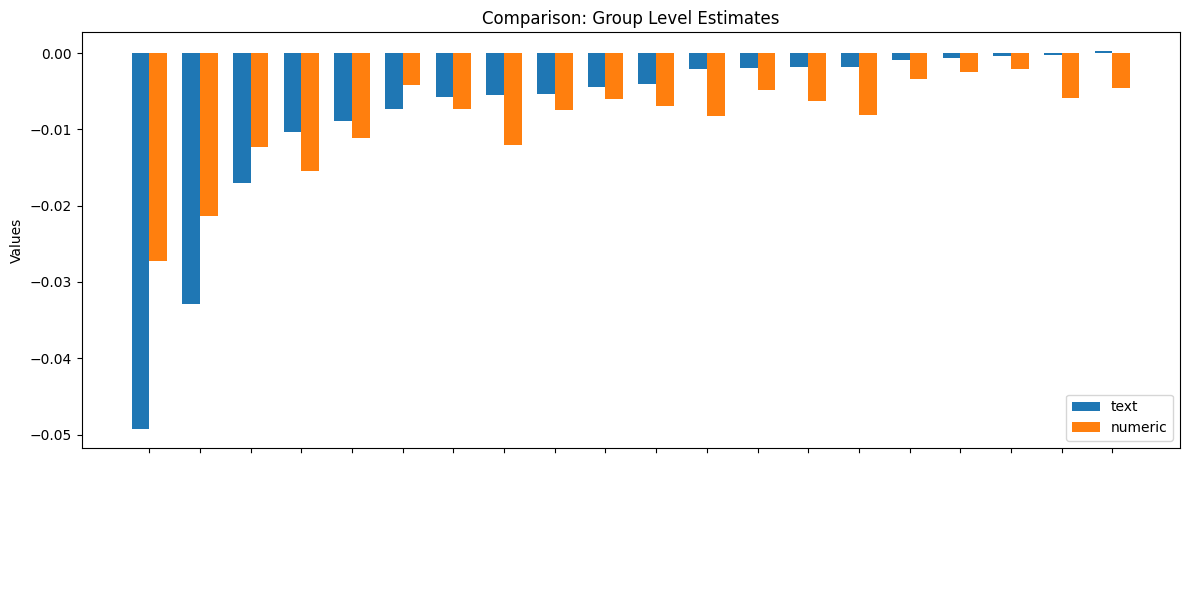}
  \caption{Group Level Comparison}
  \label{fig-group}
\end{figure}

\end{document}